\documentclass{ecai} 

\usepackage{latexsym}
\usepackage{amssymb}
\usepackage{amsmath}
\usepackage{amsthm}
\usepackage{booktabs}
\usepackage{enumitem}
\usepackage{dcolumn}
\usepackage{graphicx}
\usepackage{color}

\newtheorem{theorem}{Theorem}

\newtheorem{proposition}[theorem]{Proposition}

\newcommand{\BibTeX}{B\kern-.05em{\sc i\kern-.025em b}\kern-.08em\TeX}

\begin{document}

\newcolumntype{d}[1]{D{.}{.}{#1}}
\newcolumntype{e}[1]{D{±}{\pm}{#1}}


\begin{frontmatter}


\paperid{268} 

\title{Improving Calibration by Relating Focal Loss, Temperature Scaling, and Properness}
%


\author[A]{\fnms{Viacheslav}~\snm{Komisarenko}\thanks{Corresponding Author. Email: viacheslav.komisarenko@ut.ee.}}
\author[A]{\fnms{Meelis}~\snm{Kull}}

\address[A]{Institute of Computer Science, University of Tartu, Estonia}



\begin{abstract}
Proper losses such as cross-entropy incentivize classifiers to produce class probabilities that are well-calibrated on the training data. Due to the generalization gap, these classifiers tend to become overconfident on the test data, mandating calibration methods such as temperature scaling. The focal loss is not proper, but training with it has been shown to often result in classifiers that are better calibrated on test data. Our first contribution is a simple explanation about why focal loss training often leads to better calibration than cross-entropy training.
For this, we prove that focal loss can be decomposed into a confidence-raising transformation and a proper loss. This is why focal loss pushes the model to provide under-confident predictions on the training data, resulting in being better calibrated on the test data, due to the generalization gap. Secondly, we reveal a strong connection between temperature scaling and focal loss through its confidence-raising transformation, which we refer to as the focal calibration map. Thirdly, we propose focal temperature scaling - a new post-hoc calibration method combining focal calibration and temperature scaling. Our experiments on three image classification datasets demonstrate that focal temperature scaling outperforms standard temperature scaling.
\end{abstract}

\end{frontmatter}


\section{Introduction}
Modern machine learning systems demonstrate remarkable accuracy, often surpassing human capabilities in classification problems within textual and image domains. However, these systems still lack reliable confidence estimation in supporting their decisions, a critical issue in applications where the cost of mistakes is substantial, such as medical diagnostics, autonomous vehicle navigation, and financial decision-making.

 Neural classification models typically output a real-valued logits (scores) vector. 
 The higher the logit is, the higher the model confidence that the corresponding class should be predicted as positive.
 Quantifying model confidence in probabilities, not logits, is often more convenient. 
 
The softmax transformation is traditionally applied on top of logits to obtain predicted probabilities. Still, alternative approaches exist, including entmax and sparsemax \cite{peters2019sparse}, designed to control the sparseness of vector representations in natural language processing, and tempered exponential functions derived under convex duality assumption \cite{amid2019robust}.

Accuracy and Area Under the Curve (AUC) assess a model's discriminative power, while the Expected Calibration Error (ECE) evaluates calibration, measuring how well model confidence aligns with true class probabilities. ECE is essentially the average absolute difference between predicted probabilities and empirical accuracies across probability bins. Due to finite sample sizes, this expected value is approximated using binning, which inherently introduces approximation errors.

 The cross-entropy loss function is a common choice for training a classification model. Originating from information theory, cross-entropy measures the dissimilarity between two probability distributions, specifically, the true distribution of class labels and the predictions provided by a model. However, a notable limitation of cross-entropy trained models is their tendency to be overconfident during test time, wherein the predicted probabilities frequently exceed the actual frequencies of the corresponding classes \cite{guo2017calibration}.


Post-hoc calibration, a transformation on top of predicted logits or probabilities, is usually applied to improve the calibration of trained models \cite{silva2023classifier}. The calibration method typically involves selecting from a set of potential parametrized transformations, which are applied to adjust the model's output. This selection process is usually performed using a validation set, which helps in identifying the most effective transformation to enhance the model's calibration by adjusting the predicted probabilities to align unknown perfectly calibrated probabilities more accurately while keeping the trained model unchanged. 

Temperature scaling \cite{guo2017calibration} is a notably straightforward yet effective single-parameter calibration method. The temperature parameter $T$ adjusts the model's confidence by dividing the logits before the softmax, i.e. the last-layer activation function. Due to its simplicity and robust performance, it is often favoured over more complex calibration methods, such as matrix and vector scaling \cite{guo2017calibration}, Bayesian binning \cite{naeini2015obtaining}, and Dirichlet calibration \cite{kull2019beyond}. 
Calibration methods in deep learning have been reviewed by \citep{wang2023calibration}.

Cross-entropy is a proper loss function, meaning that it is minimized when the predicted probabilities match the true underlying probabilities of the classes, thereby incentivizing calibrated probability forecasts on the training set.
Focal loss, parametrized by positive $\gamma$, is a recent alternative to cross-entropy, with enhanced calibration on the test set, and comparable or even higher accuracy \cite{lin2017focal,mukhoti2020calibrating,wang2022calibrating}. Focal loss is not proper, but somehow, it outperforms proper losses such as cross-entropy. 

We were driven by the question of what is special in the focal loss expression that explains its calibration performance, why it could be beneficial to intentionally push predicted probabilities away from the true underlying class probabilities (because that's what focal loss does) and still receive better calibration on the test set. 







In this study, our contributions are threefold:
\begin{itemize}

    \item We deconstruct the focal loss into a composition of a confidence-raising transformation (which we call the focal calibration map) and a proper loss. We show how this decomposition explains why focal loss encourages models to generate under-confident predictions on the training data, yielding better calibration on test data. 
    \item We demonstrate a notable link between temperature scaling and focal calibration maps. More specifically, for the binary case, we show that the focal calibration map parametrized with $\gamma$ could be bounded between temperature scaling maps with $\frac{1}{T}=\gamma+1$ and $\frac{1}{T}=\gamma+1-\frac{\log(\gamma+1)}{2}$ for all possible logits. For the multiclass case, we show that the focal calibration map for any $\gamma > 0$ increases the model confidence, similar to temperature scaling with $T<1$. Moreover, we conducted an experiment showing that in the binary case, the difference between focal calibration and temperature scaling is almost neglectable when their parameters are related via the expression $\frac{1}{T} \approx 0.95\cdot \gamma+0.85$. For three- and four-dimensional cases, the closest focal calibration and temperature scaling transformations have close to the linear relationship of parameters $\gamma$ and $\frac{1}{T}$ but with higher approximation error than in binary classification.
    \item We introduce focal temperature scaling — a novel post-hoc calibration technique that composes focal calibration with temperature scaling. Our experiments on three image classification datasets and several learning strategies (mainly focal or cross-entropy-based and calibration-oriented) show that focal temperature scaling enhances the model's calibration compared to standard temperature scaling.
\end{itemize}

\section{Focal loss}

The cross-entropy loss function is a default choice for many machine learning classification problems. It has roots in maximum likelihood estimation and is a well-studied loss function with the simple expression $L_{CE}(p, y)=-\sum_{i=1}^n y_i \cdot \log(p_i)$, where $p$ and $y$ denote the predicted probability vector and the true class vector, respectively, $log$ hereafter is the natural logarithm and $n$ is the number of classes.

It is a proper loss function, which implies that the model will yield almost perfect calibration on a training set. However, on a validation and test set, cross-entropy results are often overconfident \cite{guo2017calibration}.

Focal loss \cite{lin2017focal}, which could be considered a modification on top of the cross-entropy loss, has gained widespread acceptance for classification problems. 
It is defined as $L_{FL}(p, y)=-\sum_{i=1}^n y_i \cdot (1-p_i)^{\gamma} \log(p_i)$, where the multiplier $(1-p)^\gamma$ serves as a scaling factor that for $\gamma>1$ reduces the impact of easier examples (where $p$ is close to 1) while increasing the influence of difficult examples (where $p$ is far from 1) on the loss. This adjustment helps to focus the model’s learning on harder instances that are incorrectly classified, enhancing its predictive accuracy on more challenging cases.

While the focal loss is improper, it was shown to be classification-calibrated \cite{charoenphakdee2021focal}, implying that the minimizer of the focal loss always has the best possible accuracy (Bayes-optimal classifier). Moreover, the authors proved the existence of the closed-form transformation that recovers true class probability from the focal minimizer predictions. In practice, reaching perfect focal loss minimization is almost impossible due to reasons such as data noise and imperfect learning algorithms. Still, the classification-calibration gives a theoretical justification for focal loss usage when optimizing for accuracy. 

Convexity is an essential property of most commonly used loss functions such as cross-entropy and Brier score. Convexity is helpful as many convex optimization methods and properties could be used to enhance the model training process. However, some methods assume, in fact, convexity with respect to input features, which is very hard to achieve in deep learning settings. At the same time, losses are generally called convex if they are convex with respect to the predicted probability vector. 
In the following proposition, we show that the focal loss is convex with respect to the predicted probability vector. 

\begin{proposition}
 Focal loss is convex with respect to the first (prediction) argument $p$ for all $\gamma \geq 0$.
\end{proposition}
\begin{proof}
The proof is based on an analysis of the sign of focal loss's second derivative and is presented in the Appendix.
\end{proof}

The hyperparameter $\gamma \geq 0$ is usually chosen on a validation set, while the authors of the original paper suggested value $\gamma=2$ to improve accuracy. Model calibration is noticeably impacted by the choice of $\gamma$ \cite{wang2022calibrating}.

In their comprehensive study, Mukhoti et al. 
\cite{mukhoti2020calibrating} conducted an extensive experimental comparison involving cross-entropy and focal losses, along with the Brier score and Label smoothing. Their findings indicate that models trained using focal loss have improved calibration compared to those trained with cross-entropy. They recommend an optimal $\gamma$ value of 3 based on experiments conducted within 
$\gamma \in \{1, 2, 3\}$. Additionally, they observed that models trained with focal loss generally demonstrate noticeably enhanced performance in terms of area under the curve (AUC) and log-loss. The authors also noted that minimising ECE on the validation set to determine optimal temperature $T$ yielded better calibration performance than minimizing the log-loss. 

Temperature scaling can further improve the calibration of models trained with focal loss. The improvement is often minor as the optimal temperature is frequently close to 1 \cite{mukhoti2020calibrating}. This contrasts with models trained using cross-entropy, which significantly benefit from temperature scaling with an optimal temperature often reported to be in the range of $[1.5, 3]$ \cite{guo2017calibration}.

Interestingly, optimal $\gamma$ for log-loss minimization may differ from the value that minimizes the ECE score or accuracy \cite{mukhoti2020calibrating,wang2022calibrating}. Furthermore, $\gamma$ is also extremely dataset-specific, e.g. $\gamma=10$ was the best on both log-loss and ECE in \cite{wang2022calibrating}.

In this paper, we do not specifically address the issue of class imbalance. However, it is worth noting that focal loss has been reported to perform effectively in scenarios characterized by class imbalance also \cite{wang2022calibrating}.


\section{Proper deconstruction of the focal loss}

\subsection{Motivation}

Proper losses are minimized by the predictions equal to ground-truth distribution. To illustrate it, let us consider a binary classification example. Let there be a point in feature space, meaning a fixed feature combination, which occurs 100 times in a training set with 80 actual positives and 20 actual negatives. What should the model predict for this feature combination? The model will output the number that minimizes the loss, meaning $\hat{p}=argmin_{p \in [0, 1]} (\sum_{i=1}^{20} L(p, 0) + \sum_{i=21}^{100} L(p, 1))$. The great thing about proper losses is that they will always predict $\hat{p}=0.8$ for this situation, which is an intuitively right thing to do. It makes properness a desirable property for loss functions.

Optimizing for the improper focal loss will result in the prediction of a lower value of around 0.62 for $\gamma=2$. Moreover, by varying parameter $\gamma>0$, we could receive any value between $0.8$ and $0.5$ (decision boundary). Despite this counterintuitive decision, the models trained with focal loss produce quite well-calibrated probabilities and high accuracy on the test set, in contrast with proper losses such as cross-entropy. 

We aim to identify and locate part of the focal loss expression that stands for pushing predicted probabilities away from perfectly calibrated ones. We believe that the remaining part of the focal loss expression could be connected with properness as during test time, models trained with focal loss produces quite well-calibrated predictions.





\subsection{Decomposition in the binary case}
Let $p$ be the instance's ground-truth probability of being positive, and $q$ be the model's predicted probability for class 1. Let the loss on each instance be defined as $L(q, y)$, where $y\in\{0, 1\}.$ The conditional risk, which is an expected loss under the ground-truth distribution, is defined as $R(p, q) = p \cdot L(q, 1) + (1-p) \cdot L(1-q, 0)$. The loss is called (strictly) proper if for any ground-truth distribution, the associated conditional risk is (uniquely) minimized by the prediction equal to the ground truth:

\begin{equation*}
    \forall p\in [0, 1]: argmin_{q\in[0, 1]} R(p, q) = p
\end{equation*}

In the following proposition, we deconstruct the binary focal loss into a composition of a proper loss and a monotonic probability transformation (which we refer to as the focal calibration map):

\begin{proposition}
 Let $L(q, y)$ be a binary focal loss parametrized with some $\gamma > 0$. Then it can be decomposed into a bijective function $\hat{p}: [0, 1]\rightarrow [0, 1]$ (which could be seen as a fixed calibration map) and a proper loss $L^*: [0,1] \times \{0, 1\} \rightarrow \mathbb{R}^+$ such that $L^*(q, y)=L(\hat{p}^{-1}(q), y)$ 
 and $\hat{p}$ is defined as:
\begin{gather*}
    \hat{p}(q) = \frac{1}{1 + \bigl(\frac{1-q}{q}\bigr)^{\gamma} \frac{(1-q)-\gamma q \cdot \log(q)}{q-\gamma (1-q) \cdot \log(1-q)}}
\end{gather*}
\end{proposition}
\begin{proof}
The full proof is presented in the Appendix, first, showing that $\hat{p}(\cdot)$ is a bijection and then showing the properness of $L^*$.
The proof is using known facts from the theory of binary composite losses \cite{reid2010composite}.
\end{proof}

This decomposition suggests that training a focal loss is, in fact, equivalent to training a specific proper loss, which is applied on top of an additional fixed calibration layer in the network. 
The decomposition does not change the training process because neither the predicted probabilities (before the derived calibration layer) nor the loss value is changed, which implies that backward gradients will remain the same.



The proper part of the binary focal loss decomposition could be presented via the inverse (which exists because $\hat{p}_{\gamma}(q)$ was shown to be a bijection in the proof of Proposition 2) of the focal calibration map:

\begin{equation*}
    L^{*}(q, y=1) = -(1-\hat{p}^{-1}(q))^{\gamma} \cdot \log(\hat{p}^{-1}(q))
\end{equation*}

Due to the complex composition of exponential, logarithm and power functions in the focal calibration expression $\hat{p}_{\gamma}(q)$, its inverse, and consequently, the "properized" binary focal loss $L^*$, cannot be represented as a closed-form expression. Still, it is possible to approximate and tabulate these values numerically.


The binary focal calibration map and the proper part of the focal loss (on the actual positive instance) are shown in Fig.~\ref{loss_and_link}. It could be seen that the focal calibration has a traditional sigmoid-like shape for $\gamma>0$ similarly to temperature scaling. Technically, negative $\gamma$ could also be plugged into the calibration map, yielding an "inverse"-sigmoid shape on the opposite side of the diagonal shown on the right side in Fig.~\ref{loss_and_link}. The proper part of the focal loss compared with a regular focal loss with the same $\gamma$ has lower loss values for predictions lower than 0.5 and higher loss otherwise, which could also be inferred from the calibration map shape.

\begin{figure}[t]
\centering
\centerline{\includegraphics[width=0.5\textwidth,trim={1.8cm 0cm 1.6cm 0.8cm}, clip]{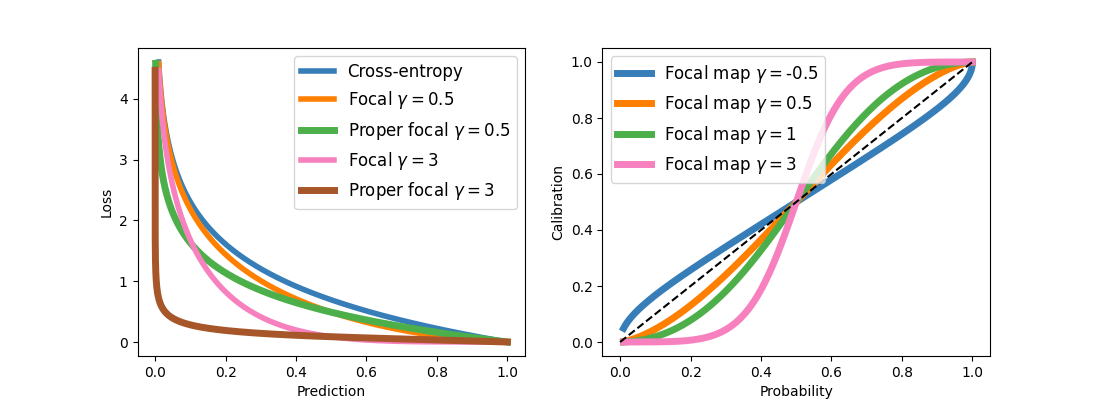}}
\caption{Left: proper part of the focal loss compared to standard focal loss with parameters $0.5$ and $3$ and cross-entropy on the true positive instance. Right: Binary focal calibration for different $\gamma$ parameters for all possible predicted probabilities.}
\label{loss_and_link}
 \vspace{0.5cm}
\end{figure}

As the focal calibration map has a similar to temperature scaling sigmoid-like shape, we approximated the temperature scaling with the focal calibration map with the parameter $\gamma$, minimising the maximum absolute difference between these calibration maps. Essentially, we ran a Linear Regression to fit these two transformations. It can be seen in Fig.~\ref{binary_fit} with the linear fit $\frac{1}{T} \approx 0.95 \cdot \gamma + 0.85$ between the inverse temperature and $\gamma$ parameter. Moreover, the maximum approximation error is marginal, lower than $1e^{-3}$ in the probability scale.

\begin{figure}[t]
\centering
\centerline{\includegraphics[width=0.5\textwidth,trim={1.8cm 0cm 1.6cm 0.8cm}, clip]{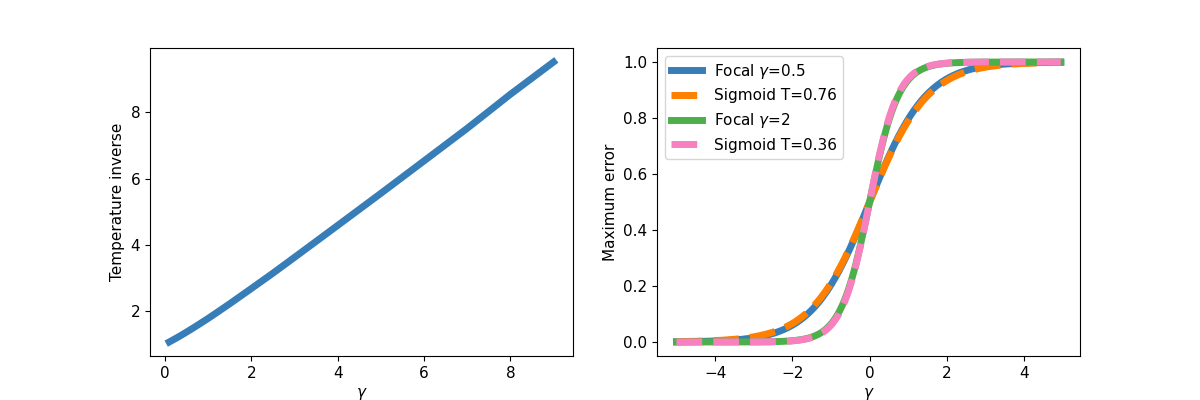}}
\caption{Left: relationship between focal calibration with parameter $\gamma$ and temperature scaling parameter $\frac{1}{T}$ fitted to minimize the maximum deviation (over all logits) of these two transformations. Right: focal calibration with parameters $0.5$ and $3$ visualized in logit scale together with closest temperature scaling maps (with $T=0.76$ and $T=0.36$ correspondingly).}
\label{binary_fit}
  \vspace{0.6cm}
\end{figure}

Given a surprisingly small approximation error of the linear fit, we tried to find lower and upper bounds for the focal calibration map using temperature scaling transformations, see the following Proposition 3:

\begin{proposition}
    Let $FC(s)=\hat{p}_{\gamma}(\frac{1}{1+e^{-s}})$ be a focal calibration function applied on top of sigmoid with a logit $s$. 
    Then, the focal calibration could be bounded between two temperature scaling maps with $T=\frac{1}{\gamma+1}$ and  $T=\frac{1}{\gamma+1 - \frac{\log(\gamma+1)}{2}}$ such that 
    \begin{align*}
    \forall s < 0 \quad \frac{1}{1+e^{-\frac{s}{\gamma+1}}} > FC(s) > \frac{1}{1+e^{-\frac{s}{\gamma+1 - \frac{\log(\gamma+1)}{2}}}} \\
    \forall s \geq 0: \quad \frac{1}{1+e^{-\frac{s}{\gamma+1}}}< FC(s) < \frac{1}{1+e^{-\frac{s}{\gamma+1 - \frac{\log(\gamma+1)}{2}}}}
    \end{align*}
\end{proposition}
\begin{proof}
The proof is provided in the Appendix.
\end{proof}
To evaluate the tightness of the derived bounds, we conduct an experimental assessment. For each value of $\gamma$, we systematically explore all possible temperatures $T>0$ in increments of 0.001. We assess the focal calibration for all logits within the range $(-20, 20)$, using a step size of 0.001, to determine if it resides between two temperature scaling maps. Through this method, we identify and select the maximal lower bound and the minimal upper bound, thereby establishing the tightest experimental bounds achievable.

The theoretical and experimental lower and upper temperature scaling bounds for focal calibration are shown in Fig.~\ref{exp_vs_theory}. It could be seen that the experimental bounds are considerably tighter than the theoretical, implying that the theoretical assumptions could be strengthened to suggest even tighter theoretical bounds, at least for logits in the range $(-20, 20)$. For example, the theoretical result suggests the focal calibration with $\gamma=4$ for all possible logits are bounded between temperature scaling with $T=0.2$ and $T=0.238$. If we double-check this result numerically, we could suggest even tighter bounds $T=0.206$ and $T=0.218$.

\begin{figure}[t]
\centering
\centerline{\includegraphics[width=0.5\textwidth,trim={1.8cm 0cm 1.6cm 0.8cm}, clip]{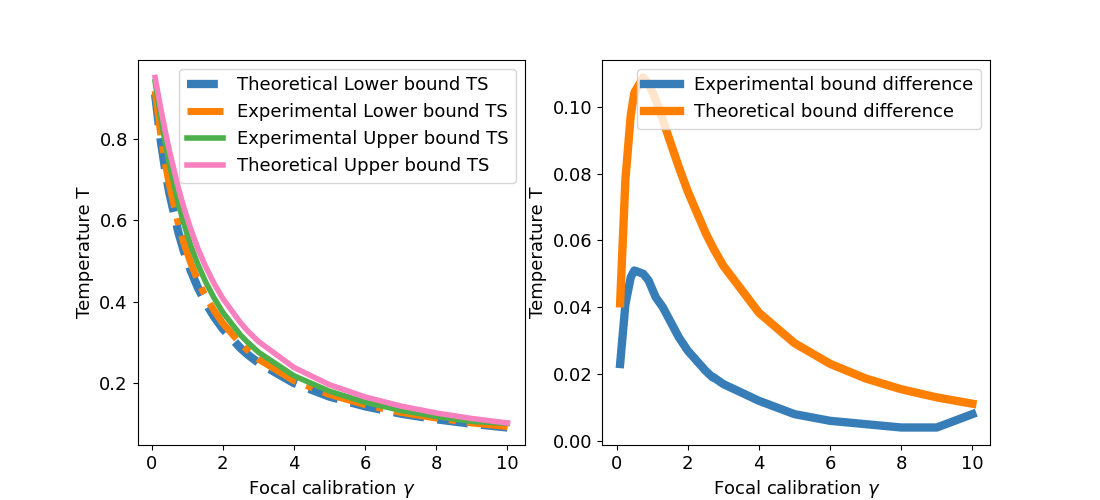}}
\caption{Left: theoretical and experimental bounds for focal calibration with temperature scaling maps. Right: theoretical and experimental width of the bounds for different $\gamma$}
\label{exp_vs_theory}
  \vspace{0.6cm}
\end{figure}



\subsection{Decomposition in the multiclass case}

Let $p=(p_1,...,p_n)$ be the ground truth vector and $q=(q_1,...,q_n)$ the prediction vector. We could consider a conditional risk $R(p, q)$, which is an expected loss under the ground-truth labels distribution $p$: $R(p, q) = \sum_{i=1}^{n} p_i \cdot L(q, y_i=1)$. The properness condition could be written as

\begin{equation*}
  argmin_{q \in \Delta^n} R(p, q) = p
\end{equation*}

Here $\Delta^n$ is a n-dimensional probability simplex, meaning that $q$ should be a valid probability distribution.

We decompose the multiclass focal loss into a monotonic probability transformation (focal calibration map) and a proper loss. The derived focal calibration map is observed to equal with the transformation that maps the outputs from focal minimizers to the true class probabilities \cite{charoenphakdee2021focal}.

\begin{proposition}
\label{prop:multiclass}
    Let $L(q, y)$ be a multiclass focal loss parametrized with some $\gamma > 0$. Then, it can be deconstructed into a composition of a bijective function $\hat{p}(q)$ and a proper loss $L^{*}(q, y)$ such that:
     \begin{align}
     \label{eq:focal}
        \hat{p}_j(q_1,...,q_n) = \frac{\frac{1}{(1-q_j)^{\gamma} \cdot (\frac{\gamma \cdot \log(q_j)}{1-q_j} - \frac{1}{q_j})}}{\sum_{k=1}^{n} \frac{1}{(1-q_k)^{\gamma} \cdot (\frac{\gamma \cdot \log(q_k)}{1-q_k} - \frac{1}{q_k})}} \quad \forall j=1..n\\
        L^{*}(q, y) = L(\hat{p_1}^{-1}(q_1,...,q_n),...,\hat{p_n}^{-1}(q_1,...,q_n)) \nonumber
\end{align}
\end{proposition}
\begin{proof}
The proof is written in the Appendix.
\end{proof}





The proper part of the multiclass focal loss is also impossible to write as a closed-form expression. 
However, we could tabulate and visualize this function using focal calibration and standard focal loss expression: $L^{*} (q_1,...,q_n)=L(\hat{p}^{-1}(q_1,...,q_n))$.

Let us consider a three-dimensional probability simplex projected onto a two-dimensional triangle to understand the derived focal calibration map and the proper part of the focal loss. Each vertex of this triangle represents a scenario in which the entire probability mass is allocated to a specific class (see Fig.~\ref{loss}).

To get more insights about the geometric shape of the proper part of the multiclass focal loss, we visualized isolines for three selected percentiles over all possible focal proper part (for $\gamma=1$ and $\gamma=3$) loss values compared with the Brier score and cross-entropy in Fig.~\ref{loss}. We presented the conditional risk associated with a specific ground-truth vector $p=(0.55, 0.3, 0.15)$. Due to properness, conditional risks for all considered losses are minimized at $q=(0.55, 0.3, 0.15)$. Notably, the Brier score displays circular isolines, reflective of its quadratic loss expression. The isolines for cross-entropy exhibit a more oval-like configuration. In contrast, the isolines for the proper part of the focal loss assume a shape more akin to a curved triangle, a trend that becomes more pronounced with increasing values of $\gamma$.


The focal calibration and corresponding closest temperature scaling maps visualized as a directional arrows map projected into a two-dimensional triangle for two different parameters $\gamma$ is shown in Fig.~\ref{activation}. Each arrow originates from a point representing an initial probability input and ends at a point corresponding to the output of the focal calibration (temperature scaling) transformation. It could be seen that arrow directions are all from the centre towards corners, which implies that the predicted distribution becomes sharper, meaning that the highest predicted probability becomes higher while all others - become lower. Moreover, the higher $\gamma$ is, the sharper the effect becomes. 

This confidence-raising effect is similar to a temperature scaling calibration with a temperature parameter lower than one $T<1$, meaning the highest predicted class probability is mapped into an even higher probability.

Next we prove this as a proposition:

\begin{proposition}
    Let the model's output for a test instance be a predicted probability vector $q=(q_1,...,q_n)$ such that the predicted class is $j$: $j = argmax_{1 \leq i \leq n} (q_i)$. Then, if we apply focal calibration to the prediction $q$, the $j$-th coordinate of the resulting vector will be higher than or equal to all coordinates of the initial prediction $q$: $\hat{p}(q)_j \geq max_{1 \leq i \leq n} (q_i)=q_j$.
\end{proposition}
\begin{proof}
The proof is listed in the Appendix.
\end{proof}
\begin{figure}[t]
\centering
\centerline{\includegraphics[width=\columnwidth]{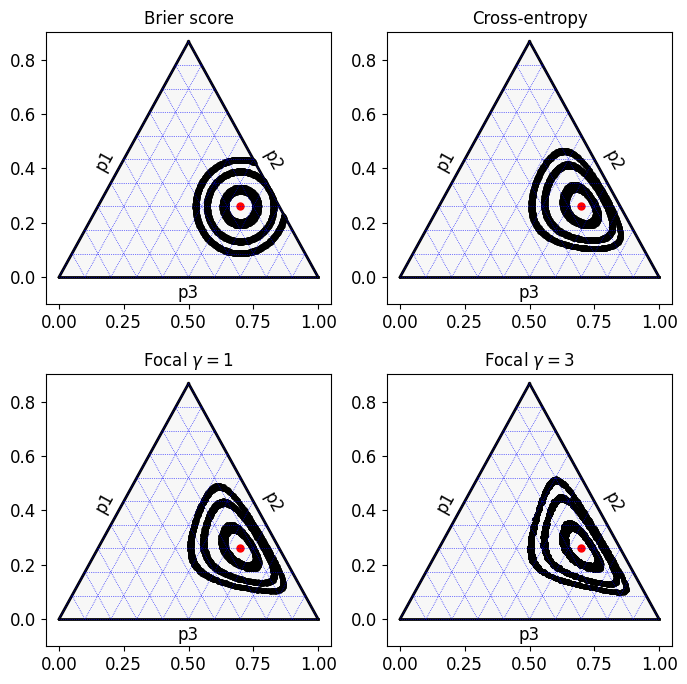}}
\caption{Brier score, cross-entropy and properized focal loss ($\gamma=1,3$) conditional risk isolines (defined by loss percentiles $3\%,12\%,20\%$) for ground truth probability $p=(0.55, 0.3, 0.15)$}
\label{loss}
  \vspace{0.6cm}
\end{figure}

Similarly to the binary case, we quantify the similarity between multiclass focal calibration and temperature scaling by finding the temperature's inverse $\frac{1}{T}$ for each $\gamma$ (considering the step of 0.001 for both $T$ and $\gamma$ over the range $(0, 10)$ and step of 0.01 for each logit in range $(-5, 5)$) that minimizes the largest difference between focal calibration and temperature scaling transformations over a three-dimensional probability simplex. 

The relationship between matching temperature inverse $\frac{1}{T}$ and $\gamma$ is shown in Fig.~\ref{temperature}. We could see a close to a linear trend that could be approximated with a function $\frac{1}{T} \approx 0.64\cdot \gamma + 0.91$.
For all positive $\gamma$, the temperature is lower than one, which verifies that the focal calibration has a sharpness effect. Moreover, lower temperature, which means higher sharpness, corresponds to higher $\gamma$. We should note that temperature scaling is not equivalent to focal calibration: the approximation error in a three-dimensional case rises from 0 (identity transformation for $\gamma=0$) to $8\%$ (of the probability scale) for $\gamma=9$. This means that the focal calibration is more complex and cannot be reliably approximated with a single simple scaling parameter of temperature scaling. 

We considered the difference between focal calibration for $\gamma=[0.5, 3]$ and the correspondent closest temperature scaling mapping for all points of the probability simplex shown in Fig.~\ref{diff}. The highest difference between these two transformations could be seen near the corners, while it non-uniformly decreases towards the centre.

We also considered this approximation in a four-dimensional case and received a similar linear dependency trend: $\frac{1}{T} \approx 0.47 \cdot \gamma + 0.85$ with similar maximum approximation error of up to $8\%$ on probability scale for larger $\gamma$ as in the three-dimensional scale. 
\section{Proper decomposition implications}
\subsection{Generalization from the training set to the test set}
The built-in focal calibration properties could hint towards the calibration performance of the focal loss function. For example, proper losses, such as cross-entropy, have almost perfect calibration performance on the training set but are generally overconfident on the validation set \cite{guo2017calibration}. A typical level of overconfidence could approximately be quantified with a temperature of 2-2.5 \cite{guo2017calibration}, which is an approximate temperature needed to achieve calibration on a validation set via temperature scaling.

It hints towards an idea of the dataset-specific "generalization compensation", meaning no matter what loss function was used for training, a post-hoc calibration on a validation set is crucial for compensating the generalization effect from training to validation set. 

If we assume that the focal loss has a similar "generalization compensation" as the cross-entropy and recall that the focal calibration map is close to temperature scaling with a temperature of about 0.4-0.5 for commonly used $\gamma$ parameters, then the composition of temperature during training (around 0.4-0.5 for focal calibration) and "generalization compensation" (roughly quantified with temperature 2-2.5) gives an overall temperature on the validation set about 1 (as multiplication of temperature 0.4-0.5 and 2-2.5), which is precisely what the temperature is for focal loss post-hoc calibration.
This explains why focal loss tends to be quite well calibrated on the test set.

\begin{figure}[t]
\centering
\centerline{\includegraphics[width=0.5\textwidth,trim={0.4cm 0cm 0.2cm 0.2cm}, clip]{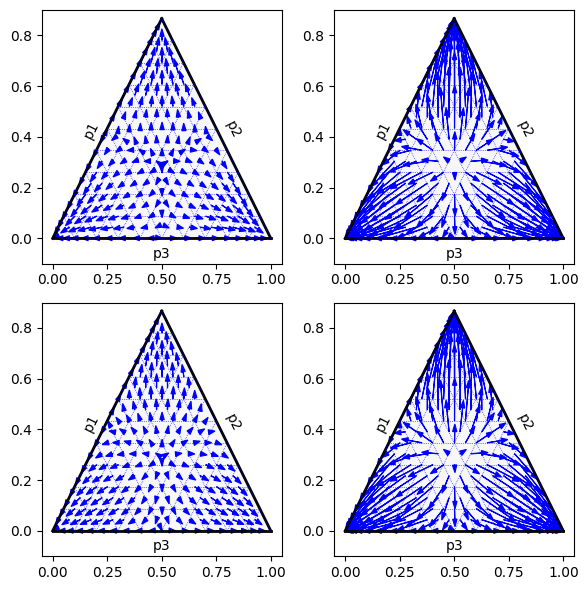}}
\caption{Focal calibration for Upper Left: $\gamma=1$, Upper Right: $\gamma=3$ parameters and temperature scaling maps for corresponding closest parameters: Bottom Left: $T=0.81$, Bottom Right: $T=0.46$ visualized as directional arrows over a uniform grid of three-dimensional probability simplex points. The focal calibration and temperature scaling maps the start of each arrow to its end. }
\label{activation}
 \vspace{0.4cm}
\end{figure}

\begin{figure}[t]
\centering
\centerline{\includegraphics[width=0.5\textwidth,trim={1.9cm 0cm 1.7cm 0.8cm}, clip]{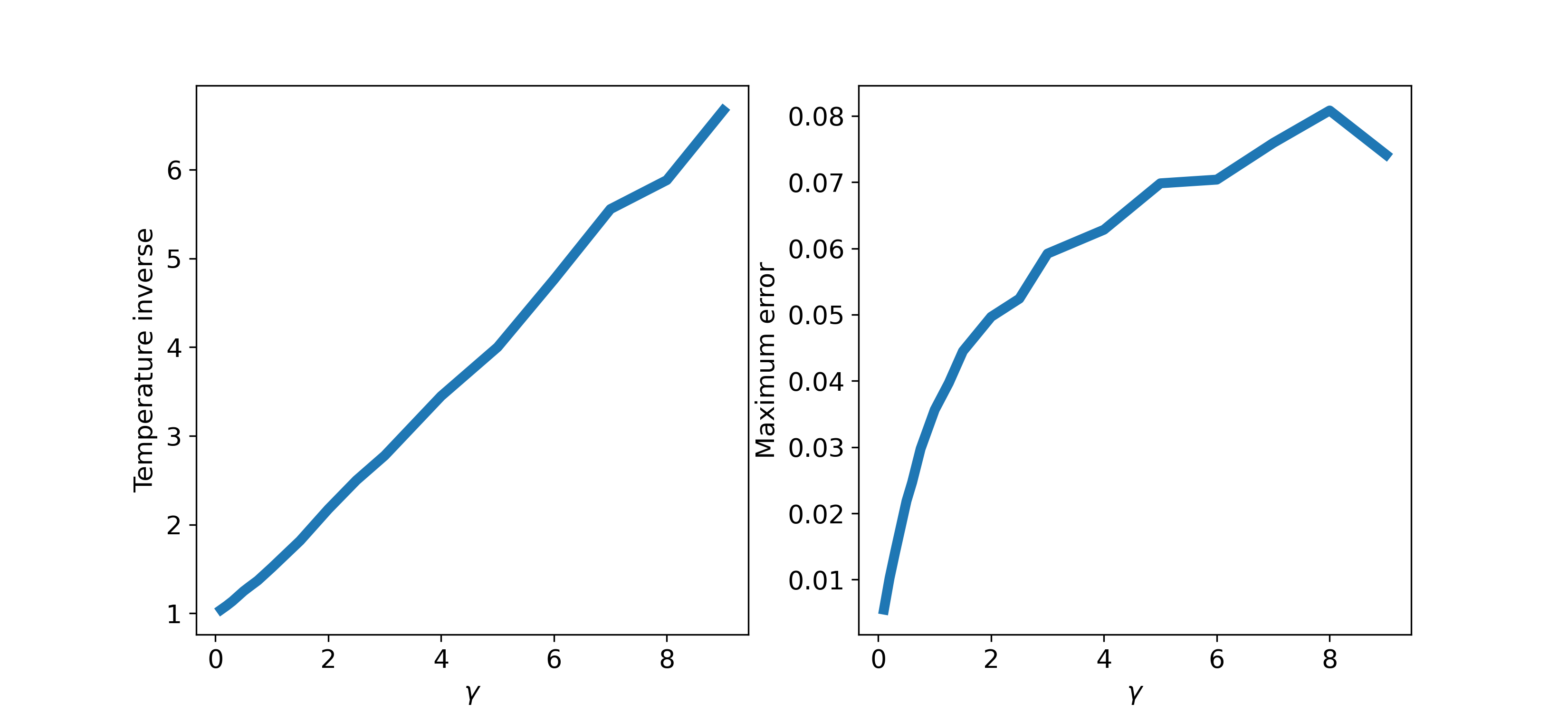}}
\caption{Left: relationship between focal calibration $\gamma$ and temperature scaling parameter $\frac{1}{T}$ chosen to minimize absolute deviation of these calibration maps for three-dimensional case. Right: dependency of maximum approximation error of these calibration maps on a focal parameter $\gamma$.}
\label{temperature}
 \vspace{0.6cm}
\end{figure}

\begin{figure}[t]
\centering
\centerline{\includegraphics[width=0.5\textwidth,trim={1.9cm 0cm 1.8cm 0.8cm}, clip]{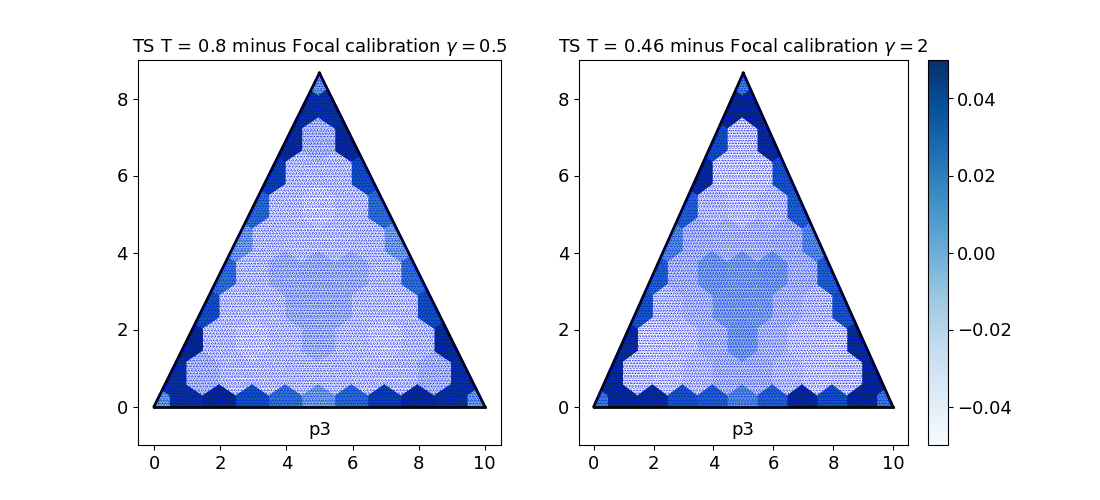}}
\caption{Temperature scaling minus focal calibration $FC(s)-TS(s)$ for the selected closest $\gamma, T$ pairs $\{(0.5, 0.8), (2, 0.46)\}$ heatmap for all points of the three-dimensional probability simplex.} 
\label{diff}
  \vspace{0.6cm}
\end{figure}


\subsection{Focal temperature scaling}

We suggest a new post-hoc calibration method called focal temperature scaling, which consists of composing focal calibration with parameter $\gamma_{ev}$ and temperature scaling with temperature $T$. 
More specifically, the method optimizes the pair of parameters $(\gamma_{ev}, T)$ on a validation set. First, instance class scores $s_1,...,s_n$ are converted into temperature-scaled softmax probabilities $q_1,...,q_n$. These probabilities are then transformed with a focal calibration map as defined by Eq.~(\ref{eq:focal}) in Proposition~\ref{prop:multiclass} parameterized by $\gamma_{ev}$, resulting in predicted probabilities $\hat{p}_1,...,\hat{p}_n$. The goal is to produce predictions that minimize the metric of interest (e.g. ECE).
As both transformations directly impact the model's confidence and are sufficiently different in a multiclass case, the focal temperature scaling should improve the model's calibration to a larger extent than standard temperature scaling. We do not limit the method to the models trained with a focal loss with parameter $\gamma_{tr}$ and, in case the training was performed with focal loss, do not require training and post-hoc calibration parameters to be equal: $\gamma_{ev} \neq \gamma_{tr}$. 

\section{Experiments}

In the experiments, we aim to verify whether applying focal temperature scaling leads to enhanced calibration compared to standard temperature scaling and suggest guidelines for practitioners.

Several hypotheses will be considered as the main focus of experiments:

\begin{itemize}
    \item Does the focal temperature scaling lead to improved calibration of the model trained with focal loss?
    \item Does the focal temperature scaling perform well when applied for other than focal losses and still receive an improved calibration compared to standard calibration methods?
    \item Is there any experimentally observed relationship between temperature scaling optimal parameter $T$ and focal loss train and calibration parameters $\gamma_{tr}, \gamma_{ev}$?

\end{itemize}

\subsection{Experiment settings}

We used an NVIDIA Tesla V100 GPU with 16 GB of VRAM for our experiments.

We consider CIFAR-10, CIFAR-100 \cite{krizhevsky2009cifar} and TinyImageNet \cite{deng2009imagenet} datasets with Resnet-50 architecture \cite{he2016deep}. We used 45000 / 5000 / 10000 images split as train, validation and test sets for CIFAR-10 and CIFAR-100 datasets. For TinyImageNet, we used 90000 / 10000 image split for train and validation sets and the TinyImageNet validation set as our test set.
Models were trained on the training set, and the validation set was used to find an optimal pair of temperature scaling and focal calibration parameters $(T, \gamma_{ev})$. A test set was used to measure the final model performance. The main evaluation metrics of interest were ECE, NLL and error rate.

The training setup was chosen to be mainly consistent with the studies \cite{mukhoti2020calibrating,ghosh2022adafocal}.
For CIFAR-10 and CIFAR-100 datasets, we started training with the SGD optimizer with a momentum of 0.9, weight decay $5e^{-4}$ and a learning rate of 0.1, which decreased to 0.01 after 150 epochs and 0.001 after 250 epochs until epoch 350. For TinyImageNet, we used the same optimizer but with a learning rate of 0.1 for the first 40 epochs, 0.01 for the next 20 epochs and 0.001 for the last 40 epochs. The batch size was 128 for CIFAR datasets and 64 for TinyImageNet. Results were evaluated on the last epoch following \cite{mukhoti2020calibrating,ghosh2022adafocal}.

We trained a model with cross-entropy loss, and more models with the focal loss with $\gamma_{train}\in\{1,2,3,5,7\}$, respectively. Also, we trained sample-dependent focal loss (FLSD-53) \cite{mukhoti2020calibrating} and AdaFocal \cite{ghosh2022adafocal} for comparison with recent studies.

During evaluation, we applied focal temperature scaling calibration on top of softmax predicted probabilities for all trained models. We iterated over a range of $[-0.5, -0.25, 0.05, 0.25, 0.37, 0.5, 0.75, 1, 5]$ for focal calibration parameter $\gamma_{ev}$ and over a range from 0.01 to 5 with a step of 0.01 for temperature parameter $T$. When presenting results for ECE,  we selected the optimal parameters $\gamma_{ev}, T$ based on the validation set ECE minimization. When considering results for log-loss, the hyperparameter choice criterion was log-loss. The results were compared with standard temperature scaling when only the parameter $T$ is selected over the same values grid and same choice criteria. An equal mass ECE measured on 15 bins was used for calibration estimate.

All code used for the implementation and experiments can be accessed in our GitHub repository \cite{komisarenko2024github}.

\subsection{Results of experiments}


\begin{table}[b!]
\caption{Test set performance for focal loss (defined by $\gamma_{tr}$), sample-dependent focal loss FLSD-53 \cite{mukhoti2020calibrating}, AdaFocal with default parameters as in \cite{ghosh2022adafocal} and cross-entropy trained models with temperature scaling versus focal temperature scaling (defined by $\gamma_{ev}$ and temperature). The CIFAR-100, CIFAR-10 and TinyImageNet datasets were used, and the results were averaged over 5 random seeds after applying temperature scaling. The mean result is reported together with the standard deviation after the $\pm$ sign. The optimal temperature is reported in brackets; temperature choice criteria were log-loss for log-loss evaluation and ECE for ECE evaluation. The best result for each metric and dataset is highlighted in bold formatting.}
\label{Cifar100}
\vskip 0.15in
\begin{center}
\begin{scriptsize}
\begin{sc}
\begin{tabular}{lrrrr} 
\hline
\multicolumn{4}{|c|}{CIFAR-100 dataset} \\ 
\hline
Approach & Accuracy & Log-loss & ECE \\
\hline
Cross-entropy & 77.6 $\pm$ 0.6 & 0.88 $\pm$ 0.02 (1.31) & 3.01 $\pm$ 0.43 (1.45) \\
 $\boldsymbol{+}\gamma_{ev}=-0.5$ & 77.6 $\pm$ 0.6 & 0.86 $\pm$ 0.02 (1.17) & 2.13 $\pm$ 0.42 (1.35) \\
\hline
Focal $\gamma_{tr}=1$ & \textbf{77.7} $\pm$ 0.3 & 0.83 $\pm$ 0.01 (1.05) & 1.66 $\pm$ 0.23 (1.15) \\
$\boldsymbol{+}\gamma_{ev}=-0.25$ & \textbf{77.7} $\pm$ 0.3 & 0.82 $\pm$ 0.01 (1.00) & 1.34 $\pm$ 0.17 (1.10) \\
\hline
Focal $\gamma_{tr}=3$ & 77.3 $\pm$ 0.5 & \textbf{0.81} $\pm$ 0.02 (0.87) & 1.28 $\pm$ 0.15 (0.91) \\
$\boldsymbol{+}\gamma_{ev}=0.05$ & 77.3 $\pm$ 0.5 & 0.82 $\pm$ 0.02 (0.87) & 1.23 $\pm$ 0.14 (0.95) \\

\hline
Focal $\gamma_{tr}=7$ & 76.3 $\pm$ 0.5 & 0.83 $\pm$ 0.01 (0.70) & 1.83 $\pm$ 0.20 (0.65)\\
$\boldsymbol{+}\gamma_{ev}=0.5$   & 76.3 $\pm$ 0.5 & 0.83 $\pm$ 0.01 (0.75) &\textbf{0.99} $\pm$ 0.07 (0.75)\\
\hline
FLSD-53 & 77.5 $\pm$ 0.5 & 0.88 $\pm$ 0.01 (1.20) & 1.89 $\pm$ 0.18 (1.27)\\
$\boldsymbol{+}\gamma_{ev}=0.25$ & 77.5 $\pm$ 0.3 & 0.88 $\pm$ 0.02 (1.05) & 1.68 $\pm$ 0.19 (1.15)\\
\hline
AdaFocal & 77.6 $\pm$ 0.2 & 0.91 $\pm$ 0.03 (1.40) & 2.96 $\pm$ 0.22 (1.52)\\
$\boldsymbol{+}\gamma_{ev}=0.25$ & 77.6 $\pm$ 0.2 & 0.93 $\pm$ 0.03 (1.48) & 2.71 $\pm$ 0.17 (1.60) \\

\multicolumn{4}{c}{} \\ 
\hline
\multicolumn{4}{|c|}{CIFAR-10 dataset} \\ 
\hline
Approach & Accuracy & Log-loss & ECE \\
\hline
Cross-entropy & \textbf{95.0} $\pm$ 0.1 & \textbf{0.16} $\pm$ 0.00 (1.59) & 1.03 $\pm$ 0.17 (1.72) \\
$\boldsymbol{+}\gamma_{ev}=1$ & \textbf{95.0} $\pm$ 0.1 & 0.17 $\pm$ 0.01 (2.20) & 0.71 $\pm$ 0.16 (2.36) \\
\hline
Focal $\gamma_{tr}=1$ & \textbf{95.0} $\pm$ 0.1 & 0.17 $\pm$ 0.01 (1.05) & 1.05 $\pm$ 0.25 (1.13) \\
$\boldsymbol{+}\gamma_{ev}=0.5$ & \textbf{95.0} $\pm$ 0.1 & 0.17 $\pm$ 0.01 (1.30) & 0.82 $\pm$ 0.35 (1.37) \\
\hline

Focal $\gamma_{tr}=3$ & 94.3 $\pm$ 0.3 & 0.19 $\pm$ 0.01 (0.75) & 1.48 $\pm$ 0.25 (0.77) \\
$\boldsymbol{+}\gamma_{ev}=5$ & 94.3 $\pm$ 0.3 & 0.20 $\pm$ 0.01 (1.83) & 0.93 $\pm$ 0.16 (1.87) \\
\hline
Focal $\gamma_{tr}=7$ & 93.1 $\pm$ 0.1 & 0.23 $\pm$ 0.01 (0.49) & 0.66 $\pm$ 0.07 (0.44)\\
$\boldsymbol{+}\gamma_{ev}=0.37$ & 93.1 $\pm$ 0.1 & 0.22 $\pm$ 0.01 (0.55) & \textbf{0.61} $\pm$ 0.10 (0.52)\\
\hline
FLSD-53 & 94.6 $\pm$ 0.1 & 0.18 $\pm$ 0.01 (1.40) & 1.30 $\pm$ 0.13 (1.40)\\
$\boldsymbol{+}\gamma_{ev}=1$ & 94.6 $\pm$ 0.1 & 0.17 $\pm$ 0.01 (1.30) & 1.23 $\pm$ 0.20 (1.33)\\
\hline
AdaFocal & 94.9 $\pm$ 0.2 & 0.18 $\pm$ 0.01 (1.58) & 1.80 $\pm$ 0.17 (1.63)\\
$\boldsymbol{+}\gamma_{ev}=5$ & 94.9 $\pm$ 0.2 & 0.20 $\pm$ 0.02 (3.70) & 0.95 $\pm$ 0.10 (3.70)\\

\multicolumn{4}{c}{} \\ 
\hline
\multicolumn{4}{|c|}{TinyImageNet dataset} \\ 
\hline
Approach & Accuracy & Log-loss & ECE \\
\hline
Cross-entropy & 49.9 $\pm$ 0.1 & 2.21 $\pm$ 0.00 (1.35) & 5.57 $\pm$ 0.31 (1.40) \\
$\boldsymbol{+}\gamma_{ev}=-0.5$ & 49.9 $\pm$ 0.1 & 2.19 $\pm$ 0.00 (1.30) & 3.66 $\pm$ 0.18 (1.40) \\
\hline
Focal $\gamma_{tr}=1$ & 50.6 $\pm$ 0.2 & 2.11 $\pm$ 0.01 (1.10) & 3.27 $\pm$ 0.13 (1.20) \\
$\boldsymbol{+}\gamma_{ev}=-0.5$ & 50.6 $\pm$ 0.2 & 2.10 $\pm$ 0.01 (1.10) & 1.86 $\pm$ 0.12 (1.18) \\
\hline
Focal $\gamma_{tr}=3$ & 51.6 $\pm$ 0.1 & 2.04 $\pm$ 0.01 (0.95) & 2.21 $\pm$ 0.14 (0.98) \\
$\boldsymbol{+}\gamma_{ev}=-0.25$ & 51.6 $\pm$ 0.1 & 2.03 $\pm$ 0.01 (0.95) & 1.63 $\pm$ 0.05 (0.97) \\
\hline
Focal $\gamma_{tr}=7$ & 50.9 $\pm$ 0.3 & \textbf{2.01} $\pm$ 0.02 (0.85) & 1.01 $\pm$ 0.02 (0.85)\\
$\boldsymbol{+}\gamma_{ev}=0.05$   & 50.9 $\pm$ 0.3 & \textbf{2.01} $\pm$ 0.02 (0.85) & \textbf{0.96} $\pm$ 0.00 (0.85)\\
\hline
FLSD-53 & \textbf{52.1} $\pm$ 0.1 & 2.02 $\pm$ 0.01 (0.95) & 2.06 $\pm$ 0.17 (0.98)\\
$\boldsymbol{+}\gamma_{ev}=-0.25$ & \textbf{52.1} $\pm$ 0.1 & 2.02 $\pm$ 0.01 (0.95) & 1.48 $\pm$ 0.22 (0.95)\\
\hline
AdaFocal & 51.6 $\pm$ 0.3 & 2.07 $\pm$ 0.03 (1.05) & 2.97 $\pm$ 0.67 (1.10)\\
$\boldsymbol{+}\gamma_{ev}=-0.5$ & 51.6 $\pm$ 0.3 & 2.07 $\pm$ 0.03 (1.05) & 1.89 $\pm$ 0.24 (1.09)\\

\bottomrule
\end{tabular}
\end{sc}
\end{scriptsize}
\end{center}
\vskip -0.1in
\end{table}

Results for all datasets are shown in Table~\ref{Cifar100}. Results are grouped per dataset. Next to each baseline approach row (cross-entropy, focal loss with different training parameters, FLSD-53, Adafocal), we put a row with the corresponding optimal parameter $\gamma_{ev}$ of the focal temperature scaling that attempts to improve the baseline approach. For each evaluation metric (Accuracy, Log-loss, ECE), we report the mean score over 5 random seeds and the standard deviation, separated with $\pm$ sign. In the brackets, we report the corresponding optimal temperature. 

The results show that focal temperature scaling could noticeably improve ECE among all baseline models while keeping the same accuracy and occasionally marginally losing on log-loss compared to standard temperature scaling.
The optimal temperature generally decreases with the growth of the focal training parameter $\gamma_{tr}$. Models trained with higher $\gamma_{tr}$ have generally better ECE but worse accuracy. Also, despite our results being mainly in line with the reported previously results \cite{mukhoti2020calibrating,ghosh2022adafocal}, we could not replicate the similar ECE score for models trained with AdaFocal despite our efforts to follow the authors' experiment guidelines closely \cite{ghosh2022adafocal} through the description provided in the original papers, supplementary materials and communication with the authors. We suspect that the reason may be in minor experiment settings, such as the exact value of the weight decay parameter. Still, the ECE scores of these methods could be improved further by applying focal temperature scaling instead of standard temperature scaling.

Moreover, we found the temperature choice criteria crucial for the experiment's results. When evaluating ECE, we could gain an extra $20\%$ relative improvement in average when optimizing for ECE on the validation set for the hyperparameter choice. For log-loss, the relative improvement is rather minor but consistent: around $1\%$.

The relationship between optimal temperature (chosen by minimizing ECE on a validation set) and focal calibration $\gamma_{ev}$ parameter for a different focal loss $\gamma_{tr}$ trained models is shown in Fig~\ref{linear}. It could be seen that the relationship is close to linear for all trained models. It implies that in practice, instead of fitting focal temperature scaling parameters $\gamma_{ev}, T$ over a 2-dimensional grid, we could iterate over fewer parameters by considering only parameters located on the line. To determine line slope and intercept, we still need to find at least two points of the line, which could be done by fixing two arbitrary $\gamma_{ev}$ and finding corresponding optimal $T$. If the standard temperature scaling will iterate over a grid of $m$ parameters, the focal temperature scaling requires $2 \cdot m$ iterations to find line coefficients and additional $m$ iterations over the line - resulting in a total of $3 \cdot m$ considered parameters, which is not that much overhead compared to standard temperature scaling. Still, in our experiments, we performed a full 2-dimensional grid search to find the best parameters and did not compare the performance of the full and optimized parameters search. The experiments with the optimized parameter search and a full understanding of this phenomenon remain for future works.


\begin{figure}[t]
\centering
\centerline{\includegraphics[width=0.5\textwidth,trim={1.8cm 0cm 1.7cm 0.8cm}, clip]{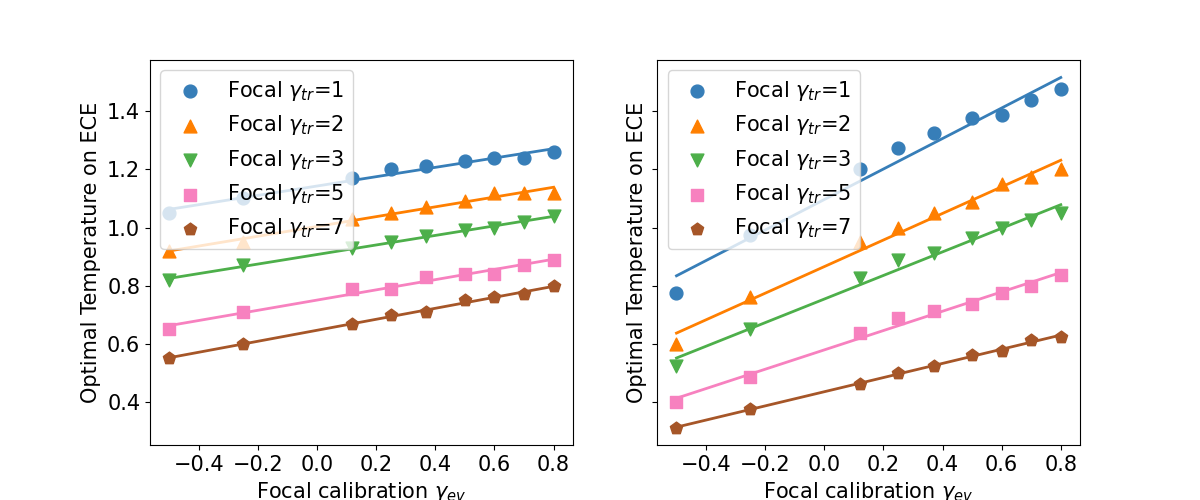}}
\caption{Relationship between optimal temperature $T$ and the focal calibration parameter $\gamma_{ev}$ chosen on ECE metric for focal loss trained models with different training parameters $\gamma_{tr}$. Left: CIFAR-100, right: CIFAR-10 dataset.}
\label{linear}
  \vspace{0.64cm}
\end{figure}

\section{Discussion and conclusions}



We derived a way to look into focal loss training as proper loss training with a confidence-raising transformation applied on top of softmax probabilities. The presented decomposition into a proper loss and a focal calibration could explain the recent successes of focal loss in calibration. The proper part of the decomposition applied on top of confidence-raised predictions pushes the probabilities to be slightly underconfident on the training set, which, due to generalization gap, leads to better calibration on the test set.  



We discovered a surprising connection between the focal calibration and temperature scaling; specifically, focal calibration could be bounded by two temperature scaling transformations with parameters $T=\frac{1}{\gamma+1}$ and $T=\frac{1}{\gamma+1-\frac{\log(\gamma+1)}{2}}$ for all logit values in the binary case. Moreover, for the multiclass case, we showed that focal calibration always increases the model's confidence, behaving similarly to temperature scaling with $T<1$. We also conducted experiments showing a close to linear dependency of parameters $\frac{1}{T}$ and $\gamma$ when minimizing deviation between these two transformations in three- and four-dimensional cases. While for the binary case, the difference between focal calibration and temperature scaling transformations for correspondent parameters is tiny, for the multiclass case, the deviation is larger and non-trivially varies across logit vectors. This implies that, for the multiclass case, composing focal and temperature scaling calibration into a single calibration method could be beneficial. We called this method focal temperature scaling, which is parametrized by focal calibration parameter $\gamma_{ev}$ and temperature $T$.

We applied the suggested focal temperature scaling method for three multiclass image classification datasets trained with cross-entropy, focal loss with different parameters and recent calibration-oriented approaches FLSD-53 and AdaFocal. The experiment results suggest a consistent improvement in calibration over all trained models compared to standard temperature scaling while keeping the accuracy unchanged by design and log-loss being approximately the same.

Moreover, the experiments showed a linear trend between optimal $\gamma_{ev}$ and $T$ parameters over different training models and datasets, suggesting that instead of fitting hyperparameters over a 2-dimensional values grid, we could iterate $\gamma_{ev}, T$ values only within the line (individual for each dataset and training method) without much performance loss. This implies that the focal temperature scaling could be applied when computational resources are limited. The detailed analysis of this dependency remains for the future work.



\begin{ack}
This work was supported by the Estonian Research Council under grant PRG1604, and by the Estonian Centre of Excellence in Artificial Intelligence (EXAI), funded by the Estonian Ministry of Education and Research.
\end{ack}







\bibliography{m268/m268}

\begin{thebibliography}{18}
\providecommand{\natexlab}[1]{#1}
\providecommand{\url}[1]{\texttt{#1}}
\expandafter\ifx\csname urlstyle\endcsname\relax
  \providecommand{\doi}[1]{doi: #1}\else
  \providecommand{\doi}{doi: \begingroup \urlstyle{rm}\Url}\fi

\bibitem[Amid et~al.(2019)Amid, Warmuth, Anil, and Koren]{amid2019robust}
E.~Amid, M.~K. Warmuth, R.~Anil, and T.~Koren.
\newblock {Robust Bi-Tempered Logistic Loss Based on Bregman Divergences}.
\newblock \emph{Advances in Neural Information Processing Systems}, 32, 2019.

\bibitem[Buja et~al.(2005)Buja, Stuetzle, and Shen]{buja2005loss}
A.~Buja, W.~Stuetzle, and Y.~Shen.
\newblock {Loss Functions for Binary Class Probability Estimation and Classification: Structure and Applications}.
\newblock \emph{Working draft, November}, 3, 2005.

\bibitem[Charoenphakdee et~al.(2021)Charoenphakdee, Vongkulbhisal, Chairatanakul, and Sugiyama]{charoenphakdee2021focal}
N.~Charoenphakdee, J.~Vongkulbhisal, N.~Chairatanakul, and M.~Sugiyama.
\newblock {On Focal Loss for Class-Posterior Probability Estimation: A Theoretical Perspective}.
\newblock In \emph{Proceedings of the IEEE/CVF Conference on Computer Vision and Pattern Recognition}, pages 5202--5211, 2021.

\bibitem[Deng et~al.(2009)Deng, Dong, Socher, Li, Li, and Fei-Fei]{deng2009imagenet}
J.~Deng, W.~Dong, R.~Socher, L.-J. Li, K.~Li, and L.~Fei-Fei.
\newblock {ImageNet: A large-scale hierarchical image database}.
\newblock In \emph{2009 IEEE Conference on Computer Vision and Pattern Recognition}, pages 248--255. IEEE, 2009.

\bibitem[Ghosh et~al.(2022)Ghosh, Schaaf, and Gormley]{ghosh2022adafocal}
A.~Ghosh, T.~Schaaf, and M.~Gormley.
\newblock {AdaFocal: Calibration-aware Adaptive Focal Loss}.
\newblock \emph{Advances in Neural Information Processing Systems}, 35:\penalty0 1583--1595, 2022.

\bibitem[Guo et~al.(2017)Guo, Pleiss, Sun, and Weinberger]{guo2017calibration}
C.~Guo, G.~Pleiss, Y.~Sun, and K.~Q. Weinberger.
\newblock {On Calibration of Modern Neural Networks }.
\newblock In \emph{International Conference on Machine Learning}, pages 1321--1330. PMLR, 2017.

\bibitem[He et~al.(2016)He, Zhang, Ren, and Sun]{he2016deep}
K.~He, X.~Zhang, S.~Ren, and J.~Sun.
\newblock {Deep Residual Learning for Image Recognition }.
\newblock In \emph{Proceedings of the IEEE Conference on Computer Vision and Pattern Recognition}, pages 770--778, 2016.

\bibitem[Komisarenko(2024)]{komisarenko2024github}
V.~Komisarenko.
\newblock {Code Repository for Improving Calibration by Relating Focal Loss, Temperature Scaling, and Properness}, 2024.
\newblock URL \url{https://github.com/slavikkom/focal_temperature_scaling}.
\newblock Accessed: 20.08.2024.

\bibitem[Krizhevsky et~al.(2009)Krizhevsky, Nair, and Hinton]{krizhevsky2009cifar}
A.~Krizhevsky, V.~Nair, and G.~Hinton.
\newblock {CIFAR-10 and CIFAR-100 datasets}.
\newblock \emph{URl: https://www.cs.toronto.edu/kriz/cifar.html}, 6, 2009.

\bibitem[Kull et~al.(2019)Kull, Perello~Nieto, K{\"a}ngsepp, Silva~Filho, Song, and Flach]{kull2019beyond}
M.~Kull, M.~Perello~Nieto, M.~K{\"a}ngsepp, T.~Silva~Filho, H.~Song, and P.~Flach.
\newblock {Beyond temperature scaling: Obtaining well-calibrated multi-class probabilities with Dirichlet calibration}.
\newblock \emph{Advances in Neural Information Processing Systems}, 32, 2019.

\bibitem[Lin et~al.(2017)Lin, Goyal, Girshick, He, and Doll{\'a}r]{lin2017focal}
T.-Y. Lin, P.~Goyal, R.~Girshick, K.~He, and P.~Doll{\'a}r.
\newblock {Focal Loss for Dense Object Detection}.
\newblock In \emph{Proceedings of the IEEE International Conference on Computer Vision}, pages 2980--2988, 2017.

\bibitem[Mukhoti et~al.(2020)Mukhoti, Kulharia, Sanyal, Golodetz, Torr, and Dokania]{mukhoti2020calibrating}
J.~Mukhoti, V.~Kulharia, A.~Sanyal, S.~Golodetz, P.~Torr, and P.~Dokania.
\newblock {Calibrating Deep Neural Networks using Focal Loss }.
\newblock \emph{Advances in Neural Information Processing Systems}, 33:\penalty0 15288--15299, 2020.

\bibitem[Naeini et~al.(2015)Naeini, Cooper, and Hauskrecht]{naeini2015obtaining}
M.~P. Naeini, G.~Cooper, and M.~Hauskrecht.
\newblock {Obtaining Well Calibrated Probabilities Using Bayesian Binning }.
\newblock In \emph{Proceedings of the AAAI Conference on Artificial Intelligence}, volume~29, 2015.

\bibitem[Peters et~al.(2019)Peters, Niculae, and Martins]{peters2019sparse}
B.~Peters, V.~Niculae, and A.~F. Martins.
\newblock {Sparse Sequence-to-Sequence Models}.
\newblock \emph{arXiv preprint arXiv:1905.05702}, 2019.

\bibitem[Reid and Williamson(2010)]{reid2010composite}
M.~D. Reid and R.~C. Williamson.
\newblock {Composite Binary Losses}.
\newblock \emph{The Journal of Machine Learning Research}, 11:\penalty0 2387--2422, 2010.

\bibitem[Silva~Filho et~al.(2023)Silva~Filho, Song, Perello-Nieto, Santos-Rodriguez, Kull, and Flach]{silva2023classifier}
T.~Silva~Filho, H.~Song, M.~Perello-Nieto, R.~Santos-Rodriguez, M.~Kull, and P.~Flach.
\newblock Classifier calibration: a survey on how to assess and improve predicted class probabilities.
\newblock \emph{Machine Learning}, pages 1--50, 2023.

\bibitem[Wang(2023)]{wang2023calibration}
C.~Wang.
\newblock {Calibration in Deep Learning: A Survey of the State-of-the-Art }.
\newblock \emph{arXiv preprint arXiv:2308.01222}, 2023.

\bibitem[Wang et~al.(2022)Wang, Balazs, Szarvas, Ernst, Poddar, and Danchenko]{wang2022calibrating}
C.~Wang, J.~Balazs, G.~Szarvas, P.~Ernst, L.~Poddar, and P.~Danchenko.
\newblock {Calibrating Imbalanced Classifiers with Focal Loss: An Empirical Study}.
\newblock In \emph{Proceedings of the 2022 Conference on Empirical Methods in Natural Language Processing: Industry Track}, pages 145--153, 2022.

\end{thebibliography}

\onecolumn
\appendix
\section{Supplementary Material}

\subsection{\textbf{Convexity of the focal loss}}
 \textbf{Proposition 1:} Focal loss is convex with respect to the first (prediction) argument $p$ for all $\gamma \geq 0$.
\begin{proof}
Let us start with the binary case. Consider special cases $\gamma=0$ and $\gamma=1$. For $\gamma=0$, focal loss equals cross-entropy, which is convex. For $\gamma=1$, the second derivative is $
        \frac{\partial^2 \bigl(-(1-p)\log(p) \bigr)}{\partial p^2} = \frac{p+1}{p^2}
    $, which is positive $\forall p \in (0, 1)$, which implies convexity.
    
    Now consider cases when $\gamma \notin \{0, 1\}$
    The first derivative of the binary focal loss with respect to the first argument is the following:
    \begin{equation*}
        \frac{\partial \bigl(-(1-p)^{\gamma}\log(p) \bigr)}{\partial p} = \gamma (1-p)^{\gamma-1}\log(p)-\frac{(1-p)^{\gamma}}{p}
    \end{equation*}
    The second derivative has a more complex expression:
    \begin{gather*}
        \frac{\partial^2 \bigl(-(1-p)^{\gamma}\log(p) \bigr)}{\partial p^2} = \frac{\gamma (1-p)^{\gamma-1}}{p}-
        \gamma(\gamma-1)(1-p)^{\gamma-2}\log(p)-
        \frac{-\gamma (1-p)^{\gamma-1}p-(1-p)^{\gamma}}{p^2}
    \end{gather*}
    Let us rewrite this expression as a quadratic expression w.r.t $\gamma$:
    \begin{gather*}
        \gamma^2 (-\log(p)(1-p)^{\gamma-2})+\gamma \bigl(\frac{2(1-p)^{\gamma-1}}{p}+  (1-p)^{\gamma-2}\log(p)\bigr)+\frac{(1-p)^{\gamma}}{p^2}
    \end{gather*}
    Trivially,  $\forall p \in (0, 1)$ all quadratic expression coefficients are non-negative (for example, linear term coefficient $\frac{2(1-p)^{\gamma-1}}{p}+ (1-p)^{\gamma-2}\log(p) = \frac{(1-p)^{\gamma-2}}{p}\bigl(2-2p+p\\log(p)\bigr) \geq 0$ because $2-2p+p\log(p)$ is monotonically decreasing on $(0, 1)$ and equals zero at the right end $p=1$), which implies that the quadratic expression is non-negative $\forall \gamma \geq 0$, which yields convexity of the focal loss.

    For the multiclass case, we will show that the Hessian matrix is positive semidefinite to prove convexity. For the focal loss, the diagonal elements are positive, as they coincide with the second derivative for the binary case, which we showed to be positive. The off-diagonal elements are all zero, as the multiclass focal loss depends only on the predicted probability of the actual class: $\frac{\partial^2 \bigl(-(1-p_i)^{\gamma}\log(p_i) \bigr)}{\partial p_j \partial p_i} = 0 \quad \forall i \neq j$. The matrix with such a structure is always positive semidefinite.
\end{proof}

\subsection{\textbf{Decomposition of the binary focal loss into proper loss and calibration map}}
\textbf{Proposition 2:}

Let $L(q, y)$ be a binary focal loss parametrized with some $\gamma > 0$. Then it could be decomposed into a bijective function $\hat{p}: [0, 1]\rightarrow [0, 1]$ (which could be seen as a fixed calibration map) and the proper loss $L^*: [0,1] \times \{0, 1\} \rightarrow \mathbb{R}^+$ such that $L^*(q, y)=L(\hat{p}^{-1}(q), y)$ 
 and the $\hat{p}$ is defined as:
\begin{gather*}
    \hat{p}(q) = \frac{1}{1 + \bigl(\frac{1-q}{q}\bigr)^{\gamma} \frac{(1-q)-\gamma q \cdot \log(q)}{q-\gamma (1-q) \cdot \log(1-q)}}
\end{gather*}

\begin{proof}

\textbf{Step 1}: let us show $\hat{p}$ is a bijection.

Let us consider partial losses of binary focal loss $L_{0}(q) = -q^{\gamma} \cdot \log(1-q)$ and $L_{1}(q)=-q^{\gamma} \cdot \log(1-q)$ and the mapping
$\hat{p}(q) = \frac{\frac{\partial L_{0}}{\partial q}(q)}{\frac{\partial L_{1}(1-q)}{\partial q} + \frac{\partial L_{0}(q)}{\partial q}}$. Partial losses are equal due to focal loss symmetry. However, for the loss on the actual positive, we are interested in the loss value w.r.t. deviation from the correct class $1-q$: $L_{1}(1-q)=-(1-q)^{\gamma} \cdot \log(q)$. 

If we insert the expression for the derivative presented in Proposition 1, we can easily verify that $\frac{\frac{\partial L_{0}}{\partial q}(q)}{\frac{\partial L_{1}(1-q)}{\partial q} + \frac{\partial L_{0}(q)}{\partial q}}$ equals to the expression stated by the proposition.

Let us show that $\hat{p}$ is a \textbf{surjection}.

The mapping is defined $\forall q \in [0, 1]$ as partial losses are differentiable. Both nominator and denominator have the same sign due to the convexity of the focal loss. Also, the denominator is never zero. 

The sum and the ratio of continuous functions is a continuous function, implying the expression $\hat{p}(q)$ is also a continuous function $\forall q\in (0, 1)$. 
Because $\hat{p}(q)$ is continuous and $\hat{p}(0) = 0$, $\hat{p}(1) = 1$, we receive that $\hat{p}(q)$ takes all possible values on $[0, 1]$ range, hence, it is a surjection.

Also, $\hat{p}$ is a \textbf{injection} because $\frac{\frac{\partial L_{0}}{\partial q}(q)}{\frac{\partial L_{1}(1-q)}{\partial q} + \frac{\partial L_{0}(q)}{\partial q}}=\frac{1}{1+\frac{\frac{\partial L_{1}(1-q)}{\partial q}}{\frac{\partial L_{0}(q)}{\partial q}}}$is monotonic on the whole $[0, 1]$ range as $\frac{\partial L_{1}(1-q)}{\partial q}$ is monotonically decreasing and $\frac{\partial L_{0}(q)}{\partial q}$ is monotonically increasing.

Finally, consider a predicted probability $q$ and the binary focal loss $L(q, y)$. Consider loss $L^{*}(\hat{p}, y)$ such that $L^{*}(\hat{p}, y) = L(q, y)$ (or, equivalently, $L^{*}(q, y) = L(\hat{p}^{-1}(q), y)$). Let us show the properness of the loss $L^{*}(\hat{p}, y)$.
For that, it is enough to show the existence of $w:[0, 1] \rightarrow \mathbb{R}^+$, such that $\int_{\epsilon}^{1-\epsilon} w(s)ds < \infty$ $\forall \epsilon > 0$ and partial losses satisfy the following expression: (this is the equivalent condition to properness as shown in Proposition 9 in \cite{buja2005loss}):
\begin{align*}
    \frac{\partial L_{0}^{*}(p)}{\partial p} = w(p)\cdot p \\
    \frac{\partial L_{1}^{*}(1-p)}{\partial p} = w(p)\cdot (1-p)
\end{align*}

Let us introduce $w(p) = \frac{\partial L^{*}_{1}(1-p)}{\partial p} + \frac{\partial L_{0}^{*}(p)}{\partial p}$, we could note that $w(p) > 0$ due to partial loss monotonicity and $w(p)$ is finite for all $p$ on $(0, 1)$.

The transformation $\hat{p}(q)$ could be simplified using the chain rule:

\begin{align*}
\hat{p}(q) = \frac{\frac{\partial L_{0}}{\partial q}(q)}{\frac{\partial L_{1}(1-q)}{\partial q} + \frac{\partial L_{0}(q)}{\partial q}} = 
\frac{\frac{\partial L_{0}^{*}}{\partial p}(p) \cdot \frac{\partial p}{\partial q}}{\frac{\partial L_{1}^{*}(1-p)}{\partial p}\frac{\partial p}{\partial q} + \frac{\partial L_{0}^{*}(p)}{\partial p}\frac{\partial p}{\partial q}}
= 
\frac{\frac{\partial L_{0}^{*}}{\partial p}(p)}{\frac{\partial L_{1}^{*}(1-p)}{\partial p} + \frac{\partial L_{0}^{*}(p)}{\partial p}}
\end{align*}

Here we used the fact that $L(q, y) = L^*(\hat{p}(q), y)$, which implies the chain rule relationship of derivatives: $\frac{\partial L(q, y)}{\partial q} = \frac{\partial L^*(\hat{p}, y)}{\partial \hat{p}} \frac{\partial \hat{p}}{\partial q}$.

If we plug in $w(p) = \frac{\partial L^{*}_{1}(1-p)}{\partial p} + \frac{\partial L^{*}_{0}(p)}{\partial p}$ and the simplified expression $\hat{p}(q)$ to the expression $w(p) \cdot p$ we receive the following:

\begin{align*}
w(p) \cdot \hat{p}(q) = & \left(\frac{\partial L^{*}_{1}(1-p)}{\partial p} + \frac{\partial L_{0}^{*}(p)}{\partial p}\right) \cdot  \frac{\frac{\partial L_{0}^{*}}{\partial p}(p)}{\frac{\partial L_{1}^{*}(1-p)}{\partial p} + \frac{\partial L_{0}^{*}(p)}{\partial p}} =  \frac{\partial L_{0}^{*}(p)}{\partial p}
\end{align*}

Similarly $w(p) \cdot (1-\hat{p}(q)) = \frac{\partial L_{1}^{*}(1-p)}{\partial p}$, this means that $L^{*}$ is a proper loss.

Finally, if we plug in the focal partial losses, we receive after simplifications:

\begin{align*}
    \hat{p}(q) = \frac{1}{1 + \bigl(\frac{1-q}{q}\bigr)^{\gamma} \frac{(1-q)-\gamma q \cdot \log(q)}{q-\gamma (1-q) \cdot \log(1-q)}}
\end{align*}
\end{proof}

\subsection{\textbf{Lower and upper bounds of the binary focal loss}}
\textbf{Proposition 3: } Let $FC(s)=\hat{p}_{\gamma}(\frac{1}{1+e^{-s}})$ be a focal calibration function applied on top of sigmoid with a logit $s$. 
    Then, the focal calibration could be bounded between two temperature scaling maps with $T=\frac{1}{\gamma+1}$ and  $T=\frac{1}{\gamma+1 - \frac{\log(\gamma+1)}{2}}$ such that 
    \begin{align*}
    \forall s < 0 \quad \frac{1}{1+e^{-\frac{s}{\gamma+1}}} > FC(s) > \frac{1}{1+e^{-\frac{s}{\gamma+1 - \frac{\log(\gamma+1)}{2}}}} \\
    \forall s \geq 0: \quad \frac{1}{1+e^{-\frac{s}{\gamma+1}}}< FC(s) < \frac{1}{1+e^{-\frac{s}{\gamma+1 - \frac{\log(\gamma+1)}{2}}}}
    \end{align*}
\begin{proof}

    Recalling the expression for $\hat{p}(q)$ from Proposition 2 and dividing the nominator and denominator by $q$ we receive the following: 

    \begin{align*}
        \hat{p}(q)=
        \frac{1}{1 + \bigl(\frac{1-q}{q}\bigr)^{\gamma} \frac{(1-q)-\gamma q \cdot \log(q)}{q-\gamma (1-q)\log(1-q)}} 
        = \frac{1}{1 + \bigl(\frac{1-q}{q}\bigr)^{\gamma} \frac{\frac{1-q}{q}-\gamma \cdot \log(q)}{1-\gamma \frac{1-q}{q} \log(1-q)}}
    \end{align*}
    Letting $q$ be the result of sigmoid transformation of the score $s$: $q=\frac{1}{1+e^{-s}}$ we could note that $\frac{1-q}{q}=e^{-s}$, $\log(q)=-\log(1+e^{-s})$ and $\log(1-q)=-\log(1+e^{s})$ and the expression simplifies to:
    \begin{equation*}
        FC(s)=\hat{p}(\frac{1}{1+e^{-s}}) = \frac{1}{1 + e^{-\gamma s} \left( \frac{e^{-s}+\gamma \\log(1 + e^{-s}) }{1 + \gamma e^{-s} \\log(1 + e^s)} \right)}
    \end{equation*}

    Trivially $FC(0) = \frac{1}{1+\frac{1+\gamma \log(2)}{1+\gamma \log(2)}} = 0.5$

    Let us find temperature scaling bounds for focal calibration. For this, we rearrange terms in focal calibration expression:
    \begin{equation*}
        \frac{1}{1 + e^{-\gamma s} \left( \frac{e^{-s}+\gamma \\log(1 + e^{-s}) }{1 + \gamma e^{-s} \\log(1 + e^s)} \right)} =  \frac{1}{1+e^{-\gamma s -s} \left( \frac{\gamma e^s \\log(1 + e^{-s}) + 1}{1 + \gamma e^{-s} \\log(1 + e^s)} \right)}
    \end{equation*}
    
It is enough to suggest bounds for expression $f_{\gamma}(s)=\left( \frac{\gamma e^s \\log(1 + e^{-s}) + 1}{1 + \gamma e^{-s} \\log(1 + e^s)} \right)$ with exponential functions $e^{c_{upper}s}, e^{c_{lower}s}$, then the overall bounds for focal calibration will be $\frac{1}{\gamma+1-c_{lower}}$ and $\frac{1}{\gamma+1-c_{upper}}$.

Firstly, let us show $f_{\gamma}(s) > e^{0} = 1$ for $s>0$ and $f_{\gamma}(s) < e^{0} = 1$ for $s < 0$:

Trivially $f_{\gamma}(s) > 0$ and $f_{\gamma}(0) = 1$. Also, the nominator is monotonically increasing (because $e^s \\log(1 + e^{-s}) = \log((1+e^{-s})^{e^s})$ is a composition of monotonically increasing functions: logarithm, $(1+\frac{1}{x})^x$ and the exponent), and the denominator monotonically decreasing, which implies that this expression is monotonically increasing. It implies that $e^{0}=1$ is the lower bound for $s>0$ and the upper bound for $s<0$.

Secondly, let us show $f_{\gamma}(s) < e^{\frac{\log(\gamma+1)}{2} \cdot s}$ for $s>0$ and $f_{\gamma}(s) > e^{\frac{\log(\gamma+1)}{2} s}$ for $s < 0$.

For that, keeping in mind that $f_{\gamma}(0) = 1$, it is enough to show that $max_{s \in \mathbb{R}} \frac{\partial f_{\gamma}(s)}{\partial s} < \frac{\gamma+1}{2}$.
The derivative for the $f_{\gamma}(s)$ is the following:

\begin{align*}
\frac{\partial f_{\gamma}(s)}{\partial s} = \frac{\gamma e^s \log(e^{-s}+1) - \frac{\gamma}{e^{-s}+1}}{\gamma \cdot e^{-s} \cdot \log(e^s+1)+1} -\frac{(\gamma e^s \log(e^{-s}+1)+1)(\frac{\gamma}{e^s+1} - \gamma e^{-s}\log(1+e^s))}{\left(\gamma \cdot e^{-s} \cdot \log(e^s+1)+1\right)^2}
\end{align*}

Using trivial bounds $\forall s \geq 0$: $\frac{\gamma}{e^s+1} > 0$, $\frac{\gamma}{e^{-s}+1} \geq \frac{\gamma}{2}$, $\gamma \log(2) \leq \gamma e^s \log(e^{-s}+1)=\gamma \log\bigl((1+e^{-s})^{e^s}\bigr) < \gamma \log(e) = \gamma$, $0 < \gamma e^{-s} \log(e^{s}+1) \leq \gamma \log(2)$,
we receive an upper bound for the derivative for $s\geq 0$: $\frac{\partial f_{\gamma}(s)}{\partial s} < \frac{\gamma-\frac{\gamma}{2}}{1}-\frac{(\gamma \log(2) + 1)(0-\gamma \log(2))}{(\gamma+1)^2}=\frac{\gamma^3+2\gamma^2(1+log^2(2)+\gamma(1+2\log(2))}{2(\gamma+1)^2} < \frac{\gamma+1}{2}$, because $\gamma^3+2\gamma^2(1+log^2(2)+\gamma(1+2\log(2)) - (\gamma+1)^3 = \gamma^2(2log^2(2)-1)+\gamma(2\log(2)-1)-1< 0 \quad \forall \gamma$ as $D=(4log^2(2)-4\log(2)+1) + 8log^2(2)-4= -3+12log^2(2)-4\log(2)<0$. The derivative of the function $\frac{\partial e^{\frac{\log(\gamma+1)}{2}s}}{\partial s} = \frac{\log(\gamma+1)}{2}e^{\frac{\log(\gamma+1)}{2}s} > \frac{\log(\gamma+1)}{2}$ is higher than derivative $\frac{\partial f_{\gamma}(s)}{\partial s}$ for $s>0$, which implies $f_{\gamma}(s) < e^{\frac{\log(\gamma+1)s}{2}}$.

Based on the same considerations, we can show that $f_{\gamma}(s) < e^{\frac{\log(\gamma+1)}{2}s}$ for negative logits $s<0$:
$\frac{\gamma}{e^s+1} \geq \frac{\gamma}{2}$, $\frac{\gamma}{e^{-s}+1} \geq 0$, $0 \leq \gamma e^s \log(e^{-s}+1) < \gamma \log(2)$, $\gamma \log(2) \leq \gamma e^{-s} \log(e^{s}+1) < \gamma$,
we receive an upper bound for the derivative for $s\leq 0$:
$\frac{\partial f_{\gamma}(s)}{\partial s} < \frac{\gamma \log(2) - 0}{\gamma \log(2)+1} - \frac{\frac{\gamma}{2}-\gamma}{(\gamma+1)^2} < \frac{\gamma+1}{2} \quad \forall \gamma \geq 0$ (because $-\frac{\frac{\gamma}{2}-\gamma}{(\gamma+1)^2} < \frac{\gamma+1}{4}$ due to $\gamma^3+3\gamma^2+\gamma+1 > 0 \quad \forall \gamma>0$ and $\frac{\gamma \log(2) - 0}{\gamma \log(2)+1} < \frac{\gamma+1}{4}$ as $\gamma^2(\log(2))+\gamma(1-3\log(2)) + 1 > 0 \quad \forall \gamma$)

\end{proof}

\subsection{\textbf{Decomposition of the multiclass focal loss into proper loss and calibration map}}

\textbf{ Proposition 4:} Let $L(q, y)$ be a multiclass focal loss parametrized with some $\gamma > 0$. Then, it could be deconstructed into a composition of a bijective function $\hat{p}(q)$ and a proper loss $L^{*}(q, y)$ such that:
     \begin{align*}
        \hat{p}_j(q_1,...,q_n) = \frac{\frac{1}{(1-q_j)^{\gamma} \cdot (\frac{\gamma \cdot \log(q_j)}{1-q_j} - \frac{1}{q_j})}}{\sum_{k=1}^{n} \frac{1}{(1-q_k)^{\gamma} \cdot (\frac{\gamma \cdot \log(q_k)}{1-q_k} - \frac{1}{q_k})}} \quad \forall j=1..n\\
        L^{*}(q, y) = L(\hat{p_1}^{-1}(q_1,...,q_n),...,\hat{p_n}^{-1}(q_1,...,q_n))
\end{align*}

\begin{proof}
    Let us consider a multiclass focal loss and a transformation 
    \begin{align*}
        \hat{p}_j(q_1,...,q_n) =
        \frac{\frac{1}{\frac{\partial L_{FL} (p_j, y_j=1)}{\partial p_j}}}{\sum_{i=1}^n \frac{1}{\frac{\partial L_{FL} (p_i, y_i=1)}{\partial p_i}}} = 
        \frac{\frac{1}{(1-q_j)^{\gamma} \cdot (\frac{\gamma \cdot \log(q_j)}{1-q_j} - \frac{1}{q_j})}}{\sum_{k=1}^{n} \frac{1}{(1-q_k)^{\gamma} \cdot (\frac{\gamma \cdot \log(q_k)}{1-q_k} - \frac{1}{q_k})}}
\end{align*}

Based on the same reasoning as for binary case, we can see that $\hat{p}(q)$ is a bijection. 

Let us consider a multiclass loss $L^{*}(p, y) = L(\hat{p}^{-1}(q_1,...,q_n),...,\hat{p_n}^{-1}(q_1,...,q_n))$ and show it is a proper loss.

For that, we consider a ground truth distribution $p_{true} \in \Delta^n$, predicted probability $p \in \Delta^n$ and the conditional risk associated with the loss $L^{*}$:

\begin{equation*}
    R(p, p_{true}) = \sum_{i=1}^{n} p_{i,true} L^{*}(p_i, y_i=1)
\end{equation*}

For properness, it is enough to show that $argmin_{p \in \Delta^n} R(p, p_{true}) = p_{true}$ for all $p_{true}$. Given that $L^{*}(\hat{p}(q), y) = L_{FL}(q, y)$, we will prove the equivalent $argmin_{q \in \Delta^n} \bigl(\sum_{i=1}^{n} p_{i,true} L(q_i, y_i=1) \bigr) = \hat{p}^{-1}(p_{true}) \quad p_{true}$.

To solve this constrained optimization problem, we could introduce a Lagrange multiplier and consider critical points:

    \begin{equation*}
    \left\{
      \begin{array}{ll}
    \left.\frac{\partial \bigl(\sum_{i=1}^{n} p_{i,true} L(q_i, y_i=1) - \lambda \cdot (\sum_{i=1}^{n} q_i-1) \bigr)}{\partial q_j}\right|_{q_j} = 0, \\ \quad \forall j \in \{1, \ldots, n\} \\
    \sum_{i=1}^{n} q_i = 1
    \end{array}
    \right.
    \label{eq:eq}
    \end{equation*}

Because for multiclass focal loss only the correspondent to the actual class predicted probability matters, meaning $\frac{\partial L(q_j, y_j=1)}{\partial q_i} = 0 \quad \forall j \neq i$, we could simplify the system of equation further:
    
    \begin{equation*}
    \left\{
      \begin{array}{ll}
     p_{j,true} \frac{\partial L(q_j, y_j=1)}{\partial q_j} - \lambda = 0\\ \quad \forall j \in \{1, \ldots, n\} \\
    \sum_{i=1}^{n} q_i = 1
    \end{array}
    \right.
    \label{eq:eq1}
    \end{equation*}
    
    Recalling that $\sum_{i=1}^n p_{i,true} = 1$, we could express the Lagrange multiplier as $\lambda=\sum_{i=1}^n \frac{1}{\frac{\partial L(q_i, y_i=1)}{\partial q_j}} \quad \forall j$ and the ground truth is equal to $p_{j,true}=\frac{\frac{1}{\frac{\partial L(q, y_i=1)}{\partial q_j}}}{\sum_{i=1}^n \frac{1}{\frac{\partial L(q, y_i=1)}{\partial q_j}}}=\hat{p}_j$, which implies $q_j=\hat{p}^{-1}(p_{true})$ and, therefore, $L^{*}(\hat{p}(q), y)$ is a proper loss.
    
\end{proof}

\subsection{The confidence-raising effect of the multiclass focal calibration}

\textbf{Proposition 5:} Let the model's output for a test instance be a predicted probability vector $q=(q_1,...,q_n)$ such that the predicted class is $j$: $j = argmax_{1 \leq i \leq n} (q_i)$. Then, if we apply focal calibration to the prediction $q$, the $j$-th coordinate of the obtained vector will be higher or equal to the highest coordinate of the initial prediction $q$: $\hat{p}(q)_j \geq q_j$

\begin{proof}
Let us compare the highest predicted probability $q_j$ and the correspondent probability after applying focal calibration:

\begin{align}
    q_j - \hat{p}(q)_j = q_j - \frac{\frac{1}{(1-q_j)^{\gamma} \cdot (\frac{\gamma \cdot \log(q_j)}{1-q_j} - \frac{1}{q_j})}}{\sum_{k=1}^{n} \frac{1}{(1-q_k)^{\gamma} \cdot (\frac{\gamma \cdot \log(q_k)}{1-q_k} - \frac{1}{q_k})}} = q_j\bigl(1 - \frac{\frac{1}{(1-q_j)^{\gamma} \cdot (\frac{\gamma \cdot \log(q_j)}{1-q_j} - \frac{1}{q_j})}}{q_j \sum_{k=1}^{n} \frac{1}{(1-q_k)^{\gamma} \cdot (\frac{\gamma \cdot \log(q_k)}{1-q_k} - \frac{1}{q_k})}}\bigr) = \\
    q_j\bigl( \frac{{q_j \sum_{k=1}^{n} \frac{1}{(1-q_k)^{\gamma} \cdot (\frac{\gamma \cdot \log(q_k)}{1-q_k} - \frac{1}{q_k})}} - \frac{1}{(1-q_j)^{\gamma} \cdot (\frac{\gamma \cdot \log(q_j)}{1-q_j} - \frac{1}{q_j})}}{{q_j \sum_{k=1}^{n} \frac{1}{(1-q_k)^{\gamma} \cdot (\frac{\gamma \cdot \log(q_k)}{1-q_k} - \frac{1}{q_k})}}}\bigr) = \frac{{q_j \sum_{k=1}^{n} \frac{1}{(1-q_k)^{\gamma} \cdot (\frac{\gamma \cdot \log(q_k)}{1-q_k} - \frac{1}{q_k})}} - \frac{1}{(1-q_j)^{\gamma} \cdot (\frac{\gamma \cdot \log(q_j)}{1-q_j} - \frac{1}{q_j})}}{\sum_{k=1}^{n} \frac{1}{(1-q_k)^{\gamma} \cdot (\frac{\gamma \cdot \log(q_k)}{1-q_k} - \frac{1}{q_k})}}
\end{align}

The function $f(q_k)=\frac{1}{(1-q_k)^{\gamma} \cdot (\frac{\gamma \cdot \log(q_k)}{1-q_k} - \frac{1}{q_k})} = \frac{1}{\frac{\partial \bigl(-(1-p)^{\gamma}\log(p) \bigr)}{\partial p}}$ is monotonically decreasing because the focal loss derivative is monotonically increasing due to convexity of the focal loss. Moreover, $f(q_k) \geq 0$ due to monotonicity and the limit value on the left end: $lim_{q_k\rightarrow 0^+} \bigl( \frac{1}{(1-q_k)^{\gamma} \cdot (\frac{\gamma \cdot \log(q_k)}{1-q_k}}\bigr) = 0$.
It implies that the denominator is negative $\forall q \in (0, 1)$ and $\forall \gamma \geq 0$. 
Let us show the nominator is positive. For this, we will first prove the following two statements:

\begin{enumerate}
    \item Let $f(x)$ be a concave function defined on $[0, 1]$ range such that $f(0)=0$. Then $\forall \lambda \in [0, 1]$: $f(\lambda x) \geq \lambda f(x)$.

    \begin{proof}
        $f(\lambda x) = f(\lambda x + (1-\lambda) 0) \geq \lambda f(x) + (1-\lambda)f(0) = \lambda f(x)$
    \end{proof}

    \item The function $f(q)=\frac{1}{(1-q)^{\gamma} \cdot (\frac{\gamma \cdot \log(q)}{1-q} - \frac{1}{q})}$ is concave on the whole range $\forall q \in (0, 1)$.
    \begin{proof}
        Let $L(p)=-(1-q)^{\gamma} \log(q)$ denotes a focal loss. Then, $f(q)=\frac{1}{L'(q)}$. Let us compute first and second derivatives of $f(q)$ expressed in terms of focal loss derivatives:
        $f'(q)=(\frac{1}{L'(q)})'=\frac{L^{''}(q)}{(L'(q))^2}$. Also, $f''(q)=(\frac{L^{''}(q)}{(L'(q))^2})'=\frac{L^{'''}(q)\cdot (L'(q))^2 - L^{''}(q)\cdot 2 L'(q) \cdot L^{''}(q)}{(L'(q))^4} = \frac{L'(q) \cdot (L^{'''}(q)\cdot L'(q) - 2 (L^{''}(q))^2)}{(L'(q))^4}$. 


        Let us show that the first derivative is monotonically decreasing: consider $f'(q)=(\frac{1}{L'(q)})'=\frac{L^{''}(q)}{(L'(q))^2}$. After inserting the first and second derivatives of the focal loss (which we derived in Proposition 1), we receive:

        \begin{align*}
            f'(q)=(\frac{1}{L'(q)})'=\frac{L^{''}(q)}{(L'(q))^2} = \frac{\frac{\gamma (1-q)^{\gamma-1}}{q}-
        \gamma(\gamma-1)(1-q)^{\gamma-2}\log(q)-
        \frac{-\gamma (1-q)^{\gamma-1}q-(1-q)^{\gamma}}{q^2}}{\bigl(\gamma (1-q)^{\gamma-1}\log(q)-\frac{(1-q)^{\gamma}}{q}\bigr)^2} = \\
        \frac{(1 - q)^{-\gamma} \left( - (\gamma - 1) \gamma q^2 \log(q) - (q - 1) ((2\gamma - 1) q + 1) \right)}{(\gamma q \log(q) + q - 1)^2}
        \end{align*}
        The $(1-q)^{-\gamma}$ term, which is monotonically decreasing, will dominate other terms, hence, the overall expression is monotonically decreasing, and, therefore, the $f(q)$ is concave.
    \end{proof}
\end{enumerate}

Given these two statements, and considering the function $f(q_k)=\frac{1}{(1-q_k)^{\gamma} \cdot (\frac{\gamma \cdot \log(q_k)}{1-q_k} - \frac{1}{q_k})}$, we need to prove that $q_j \sum_{i=1}^n f(q_i) \geq f(q_j)$, meaning the non-negative of the nominator, to finally prove that $q_j \leq \hat{p}(q)_j$. Using the substitution $\lambda = \frac{q_i}{q_j} \leq 1$ and applying the inequality from the first statement, we receive:
$q_j \sum_{i=1}^n f(q_i) = q_j \sum_{i=1}^n f(\frac{q_i}{q_j}q_j) \geq \sum_{i=1}^n \bigl(q_j \frac{q_i}{q_j} f(q_j) \bigr) = f(q_j) \sum_{i=1}^n q_i = f(q_j)$.

\end{proof}

\end{document}